\documentclass[a4paper,twoside]{article}

\usepackage{epsfig}
\usepackage{subcaption}
\usepackage{calc}
\usepackage{amssymb}
\usepackage{amstext}
\usepackage{amsmath}
\usepackage{amsthm}
\usepackage{multicol}
\usepackage{pslatex}
\usepackage{apalike}
\usepackage{chngcntr}
\usepackage{amssymb, amsmath, amsthm, thm-restate, upgreek, dsfont, bm}
\usepackage{SCITEPRESS}     

\usepackage{amssymb, amsmath, amsthm, thm-restate, upgreek, dsfont, bm}

\counterwithin*{equation}{section}
\newtheorem{theorem}{Theorem}[section]

\newtheorem{definition}[theorem]{Definition}
\newtheorem{lem}[theorem]{Lemma}
\renewcommand{\D}{\mathcal{D}}

\newcommand{\A}{\mathcal{A}}
\newcommand{\1}{\mathds{1}}
\newcommand{\X}{\mathrm{X}}

\newcommand{\M}{Monta{\~{n}}ez}
\newcommand{\g}{\phi}
\newcommand{\E}{\mathbb{E}}
\newcommand{\N}{\mathbb{N}}

\newcommand{\dif}{\text{d}}
\newcommand{\bias}{\mathrm{Bias}}
\newcommand{\gp}{general probability of success}

\newcommand{\GP}{General Probability of Success}
\newcommand{\ep}{expected per-query probability of success}

\newcommand{\EP}{Expected Per-Query Probability of Success}
\newcommand{\dd}{decomposable}
\newcommand{\Dd}{Decomposable}
\newcommand{\dy}{decomposability}
\newcommand{\Dy}{Decomposability}
\newcommand{\dm}{decomposable probability-of-success metric}

\newcommand{\floor}[1]{\lfloor #1 \rfloor}
\newcommand{\GG}[1]{}

\numberwithin{equation}{section}
\begin{document}

\title{\Dd{} Probability-of-Success Metrics in Algorithmic Search}
\ifx
\author{\authorname{Author 1\sup{1}\orcidAuthor{0000-0000-0000-0000}, Author 2\sup{2}\orcidAuthor{0000-0000-0000-0000}, Author 3\sup{3}\orcidAuthor{0000-0000-0000-0000},Author 4\sup{1}\orcidAuthor{0000-0000-0000-0000}, and Author 5\sup{1}\orcidAuthor{0000-0000-0000-0000}}
\affiliation{\sup{1} University 1}
\affiliation{\sup{2} University 2}
\affiliation{\sup{3} University 3}
\email{\{author\}@email.edu}
}
\fi

\author{\authorname{Tyler Sam\sup{1}\orcidAuthor{0000-0001-7974-3226}, Jake Williams\sup{1}\orcidAuthor{0000-0001-9714-1851}, Abel Tadesse\sup{2}\orcidAuthor{0000-0002-3337-9454}, Huey Sun\sup{3}\orcidAuthor{0000-0002-0949-3169}, and George \M\sup{1}\orcidAuthor{0000-0002-1333-4611}}
\affiliation{\sup{1} Harvey Mudd College, California, USA}
\affiliation{\sup{2} Claremont McKenna College, California, USA}
\affiliation{\sup{3} Pomona College, California, USA}
\email{\{tsam, jlwilliams, gmontanez\}@g.hmc.edu, aleulseged20@cmc.edu, hssa2016@pomona.edu}
}


\abstract{Previous studies have used a specific success metric within an algorithmic search framework to prove machine learning impossibility results. However, this specific success metric prevents us from applying these results on other forms of machine learning, e.g. transfer learning. We define \dd{} metrics as a category of success metrics for search problems which can be expressed as a linear operation on a probability distribution to solve this issue. Using an arbitrary \dd{} metric to measure the success of a search, we demonstrate theorems which bound success in various ways, generalizing several existing results in the literature.}

\onecolumn \maketitle \normalsize \setcounter{footnote}{0} \vfill

\section{\uppercase{Introduction}}
\label{sec:introduction}

\noindent Many machine learning tasks, such as classification, regression and clustering, can be reduced to search problems~\cite{montanez2017machine}. Through this reduction, one can apply concepts from information theory to derive results about machine learning. To compare the success of different algorithms, or the expected probability of finding a desired element, \M{} defined a metric of success that averaged the  probability of success over all iterations of an algorithm~\cite{montanez2017machine}. While this metric has many applications, it is not appropriate for cases where the probability of success for a given iteration of an algorithm is required. An example of this is transfer learning, where the probability of success at the final step of the algorithm is more relevant than the average probability of success. 

Building on this work, we define \dy{} as a property of probability-of-success metrics and show that the \ep{}~\cite{montanez2017machine} and more general probability of success metrics are \dd. We then show that the results previously proven for the \ep{}  hold for all \dm s. Under this generalization,we can prove results related to the probability of success for specific iterations of a search rather than just uniformly averaged over the entire search, giving the results much broader applicability.

\section{\uppercase{Related Work}}

\noindent Several decades ago, Mitchell proposed that classification could be viewed as search, and reduced the problem of learning generalizations to a search problem within a hypothesis space~\cite{mitchell80,mitchell1982generalization}. \M{} subsequently expanded this idea into a formal search framework~\cite{montanez2017machine}.

\M{} showed that for a given algorithm with a fixed information resource, favorable target sets, or the target sets on which the algorithm would perform better than uniform random sampling, are rare. He did this by proving that the proportion of $b$-bit favorable problems has an exponentially decaying restrictive bound~\cite{Montanez2016TheFO}. He further showed that this scarcity of favorable problems exists even for $k$-sparse target sets. 

\M{} et al.\ later defined bias, the degree to which an algorithm is predisposed to a fixed target, with respect to the expected per-query probability of success metric, and proved that there were a limited number of favorable information resources for a given bias~\cite{montanez2019fobfl}. Using the search framework, they proved that an algorithm cannot be favorably biased towards many distinct targets simultaneously.

As machine learning grew in prominence, researchers began to probe what was possible within machine learning. Valiant considered learnability of a task as the ability to generate a program for performing the task without explicit programming of the task~\cite{ValiantLearnability}. By restricting the tasks to a specific context, Valiant demonstrated a set of tasks which were provably learnable. 
    
Schaffer provided an early foundation to the idea of bounding universal performance of an algorithm~\cite{SchafferConservation}. Schaffer analyzed generalization performance, the ability of a learner to classify objects outside of its training set, in a classification task. Using a baseline of uniform sampling from the classifiers, he showed that, over the set of all learning situations, a learner's generalization performance sums to zero, which makes generalization performance a conserved quantity. 

Wolpert and Macready demonstrated that the historical performance of a deterministic optimization algorithm provides no \textit{a priori} justification whatsoever for its continued use over any other alternative going forward~\cite{Wolpert1997NoFL}, implying that there is no utility in rationally choosing a thus-far better algorithm over choosing the opposite. Furthermore, just as there does not exist a single algorithm that performs better than random on all possible optimization problems, they proved that there also does not exist an optimization problem on which all algorithms perform better than average.

Continuing the application of prior knowledge to learning and optimization, G{\"{u}}l{\c{c}}ehre and Bengio showed that the worse-than-chance performance of certain machine learning algorithms can be improved through learning with hints, namely, guidance using a curriculum~\cite{bengio}. So, while Wolpert's results might make certain tasks seem futile and infeasible, G{\"{u}}l{\c{c}}ehre's empirical results show that there exist some alternate means through which we can utilize prior knowledge to get better results in both learning and optimization.

Others have worked towards meaningful bounds on algorithmic success through different approaches. Sterkenburg approached this concept from the perspective of Putnam, who originally claimed that a universal learning machine is impossible through the use of a diagonalization argument~\cite{sterk}. Sterkenburg follows up on this claim, attempting to find a universal inductive rule by exploring a measure which cannot be diagonalized. Even when attempting to evade Putnam's original diagonalization, Sterkenburg is able to apply a new diagonalization that reinforces Putnam's original claim of the impossibility of a universal learning machine. 

There has also been work on proving learning bounds for specific problems. Kumagai and Kanamori analyzed the theoretical bounds of parameter transfer algorithms and self-taught learning~\cite{KumagaiBounds}. By looking at the local stability, or the degree to which a feature is affected by shifting parameters, they developed a definition for parameter transfer learnability, which describes the probability of effective transfer.

\subsection{Distinctions from Prior Work}
\noindent The \ep{} metric previously defined in the algorithmic search framework \cite{montanez2017machine} tells us, for a given information resource, algorithm, and target set, how often (in expectation) our algorithm will successfully locate elements of the target set. While this metric is useful when making general claims about the performance of an algorithm or the favorability of an algorithm and information resource to the target set, it lacks the specificity to make claims about similar performance and favorability on a per-iteration basis. This trade-off calls for a more general metric that can be used to make both general and specific (per iteration) claims. For instance, in transfer learning tasks, the performance and favorability of the last pre-transfer iteration is more relevant than the overall \ep. The \gp, which we will define as a particular \dm{}, is a tool through which we can make claims at such specific and relevant steps.

\section{\uppercase{Background}}
\noindent In this section, we will present definitions for the main framework that we will use throughout this paper. 

\subsection{The Search Framework}
\M{} describes a framework which formalizes search problems in order to analyze search and learning algorithms~\cite{Montanez2016TheFO}. There are three components to a search problem. The first is the finite discrete search space, $\Omega$, which is the set of elements to be examined. Next is the target set, $T$, which is a nonempty subset of the search space that we are trying to find. Finally, we have an external information resource, $F$, which provides an evaluation of elements of the search space. Typically, there is a tight relationship between the target set and the external information resource, as the resource is expected to lead to or describe the target set in some way, such as the target set being elements which meet a certain threshold under the external information resource.

Within the framework, we have an iterative algorithm which seeks to find elements of the target set, shown in Figure \ref{fig:ml-as-search}. The algorithm is a black-box that has access to a search history and produces a probability distribution over the search space. At each step, the algorithm samples over the search space using the probability distribution, evaluates that element using the information resource, adds the result to the search history, and determines the next probability distribution. The abstraction of finding the next probability distribution as a black-box algorithm allows the search framework to work with all types of search problems.

\begin{figure}[h]
    \centering
    \includegraphics[scale=0.4]{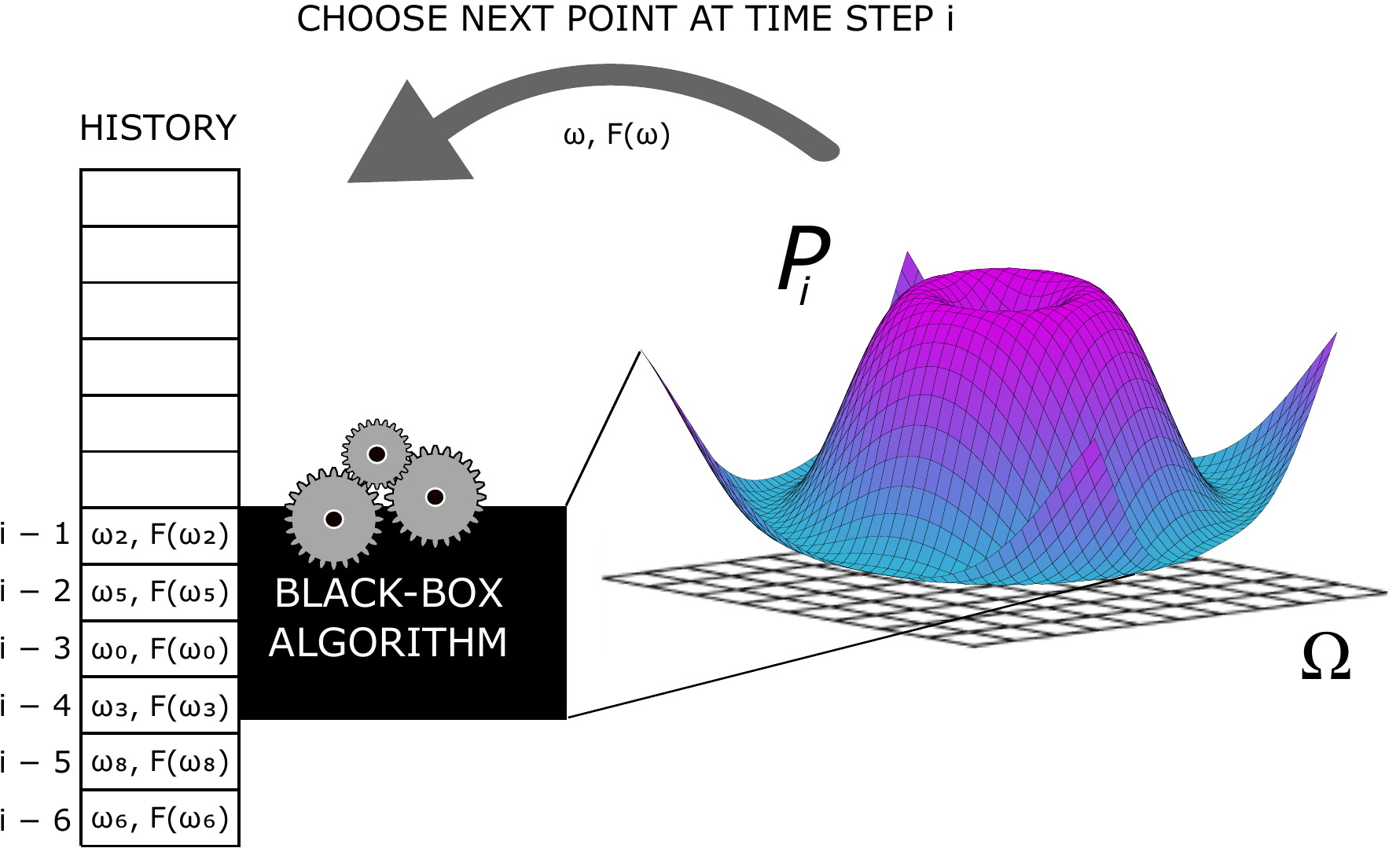}
    \caption{Black-box search algorithm. We iteratively populate the history with samples from a distribution that is determined by the black-box at each iteration, using the history \cite{montanez2017machine}.}
    \label{fig:ml-as-search}
\end{figure}

The ML-as-search framework is valuable because it provides a structure to understand and reason about different machine learning problems within the same formalism. For example, we can understand regression as a search through a space of possible regression functions, or parameter estimation as a search through possible vectors for a black-box process \cite{montanez2017machine}. Therefore, we can apply results about search problems to any machine learning problem we can cast into the search framework.

\subsection{Expected Per-Query Probability of Success}
\noindent In order to compare search algorithms, \M{} defined the \ep,
\begin{equation}
\small
  q(t, f) = \E_{\tilde{P}, H}\left [\frac{1}{|\tilde{P}|}\sum_{i=1}^{|\tilde{P}|}P_i(w \in t) \bigg| f \right] = \overline{P}(X \in t|f)
\end{equation}
where $\tilde{P}$ is the sequence of probability distributions generated by the black box, $H$ is the search history, and $t$ and $f$ are the target set and information resource of the search problem, respectively~\cite{Montanez2016TheFO}. This metric of success is particularly useful because it can be shown that $q(t, f) = \textbf{t}^{\top}\overline{\bf P}_f$, where $\overline{ \bf P}_f$ is the average of the vector representation of the probability distribution from the search algorithm at each step, conditioned on an information resource $f$. 

Measuring success using the \ep, \M{} demonstrated bounds on the success of any search algorithm~\cite{Montanez2016TheFO}. The Famine of Forte states that for a given algorithm, the proportion of target-information resource pairs yielding a success level above a given threshold is inversely related to the threshold. Thus, the greater the threshold for success, the fewer problems you can be successful on, regardless of the algorithm. The \ep{} can also be used to prove a version of the No Free Lunch theorems, demonstrating that all algorithms perform the same averaged over all target sets and information resources, as is done in Theorem~\ref{thm:NFL} of the current manuscript.

\subsection{Bias}

\noindent Using the search framework, \M{}  defined a measure of bias between a distribution over information resources and a fixed target~\cite{montanez2019fobfl}. For a distribution $\D$ over a collection of possible information resources $\mathcal{F}$, with $F \sim \D$, and a fixed $k$-hot\footnote{$k$-hot vectors are binary vectors of length $|\Omega|$ with exactly $k$ ones.} target \textbf{t}, the bias between the distribution and the target is defined as
\begin{align}
\text{Bias}(\D, \textbf{t}) &= \mathbb{E}_{\D}[\textbf{t}^{\top}\overline{\bf P}_F] - \frac{k}{|\Omega|} \\
&= \textbf{t}^{\top}\mathbb{E}_{\D}[\overline{\bf P}_F] - \frac{\|\textbf{t}\|^2}{|\Omega|}\\
&= \textbf{t}^{\top}\int_{\mathcal{F}} \overline{\bf P}_f \D(f) \text{d}f - \frac{\|\textbf{t}\|^2}{|\Omega|}.
\end{align}
Recall from above that $\overline{\bf P}_f$ was the averaged probability distribution over $\Omega$ from a search.  

The bias term measures the performance of an algorithm in expectation (over a given distribution of information resources) compared to uniform sampling. Mathematically, this is computed by taking the difference between the expected value of the average performance of an algorithm and the performance of uniform sampling. The distribution $\D$ captures what information resources (e.g., datasets) one is likely to encounter.

For a non-mathematical example of the effect of bias, suppose we are searching for parking space within a parking lot. If we randomly choose parking spaces to check, we are searching without bias. However, if we consider the location of the parking spaces, we may find that parking spaces furthest from the entrance are usually free, and could find an open parking space with a higher probability. Here, the information resource telling us the distance of each parking space from the entrance and our belief that parking spaces further from the entrance tend to be open creates a distribution over possible parking spaces, favoring those that are further away for being checked first.

\section{\uppercase{Preliminaries}}
\noindent In this section, we introduce a new property of success metrics called \dy, which allows us to generalize concepts of success and bias. We provide a number of prelimimary lemmata, with full proofs given in the Appendix.

\subsection{\Dy}
\noindent We now give a formal definition for a \dm, which will be used throughout the rest of the paper.

\begin{definition}
A probability-of-success metric $\g$ is \textbf{\dd} if and only if there exists a $\mathbf{P}_{\g,f}$ such that
\begin{equation}
    \g(t,f) = \mathbf{t}^{\top}\mathbf{P}_{\g,f}=P_{\g}(X \in t|f),
\end{equation}
where $\mathbf{P}_{\g,f}$ is not a function of $t$, being conditionally independent of it given $f$.
\end{definition}
\noindent As we stated previously, what makes the \ep{} particularly useful is that it can be represented as a linear function of a probability distribution. This definition allows us to reference any probability-of-success metric having this property.

As a first example, we show that the \ep{} is a \dm.

\begin{restatable}[\Dy{} of the \EP{} ]{lem}{DEPlem}\label{lem:DEP}
The \ep{} is \dd, namely,
     \begin{align}
q(t,f) = \mathbf{t}^\top \overline{\mathbf{P}}_f.
    \end{align}
\end{restatable}

\noindent Our goal is to show that the theorems proved for the \ep{} hold for all \dd{} metrics. Showing that the \ep{} is \dd{} suggests that these theorems may be generalizable to any metrics sharing that property.

\subsubsection{The General Probability of Success}
\noindent While the \ep{} averages the probability of success over each of the queries in a search history, we may care more about a specific query in the search history, e.g., the final query of a sequence. Thus, we can generalize the \ep{} by replacing the averaging with an arbitrary distribution $\alpha$ over the probability distributions in the search history. We define the \GP{} as
\begin{equation}
\small
q_{\alpha}(t, f) = \E_{\tilde{P}, H}\left [\sum_{i = 1}^{|\tilde{P}|}\alpha_iP_i(w \in t) \bigg| f \right] = P_{\alpha}(X \in t|f)  \end{equation}
where $P_\alpha$ is a valid probability distribution on the search space $\Omega$ and $\alpha_i$ is the weight allocated to the $i$th probability distribution in our sequence. This formula allows us to consider a wide variety of success metrics as being instances of the general probability of success metric. For example, the expected per-query probability of success is equivalent to setting $\mathbf{P}_{\alpha, f} = \overline{\mathbf{P}}_f$, with $\alpha_i = 1/|\tilde{P}|$. Similarly, a metric of success which only cares about the final query can be represented by letting $\mathbf{P}_{\alpha, f} = \overline{\mathbf{P}}_{n,f}$ where $n$ is the length of the sequence of queries, and $\overline{\mathbf{P}}_{n,f}$ is the average of the distributions from the $n$-th iteration of our search. 

It should be noted that $\alpha$ within the expectation will be random, being defined over the random number of steps within $\tilde{P}$. Our operative definition of the $\alpha$ distribution, however, will allow us to generate the corresponding distribution for the needed number of steps, such as when we place all mass on the $n$-th iteration of the search. With a slight abuse of notation, we thus let $\alpha$ signify both the process by which the distribution is generated as well as the particular distribution produced for a given number of steps.

As the \gp{} provides a layer of abstraction above the \ep{}, if we prove that results about the \ep{} also hold for the \gp{}, we gain a more powerful tool set. To do so, we must first demonstrate that the \gp{} is a \dm.

\begin{restatable}[\Dy{} of the \GP{} Metric]{lem}{DGPlem}\label{lem:DGP}
The \gp{} is \dd, namely, 
     \begin{align}
q_{\alpha}(t, f) = \mathbf{t}^{\top}\mathbf{P}_{\alpha,f}.
    \end{align}
\end{restatable}

\noindent These lemmata allow us to apply later theorems about \dd{} metrics to these two useful metrics. Given a metric of interest, performing a similar proof of \dy{} will allow for the application of the subsequent theorems.

\begin{restatable}[\Dy{} closed under  expectation]{lem}{DCElem}\label{lem:DCE}
Given a set $S = \{\g_i\}$ of \dm s and a distribution $\mathcal{D}$ over $S$, it holds that
\begin{equation}
    \g'(t,f) = \E_{\mathcal{D}}[\g(t,f)]
\end{equation}
is also a \dm.
\end{restatable}

\noindent Lemma~\ref{lem:DCE} gives us an easy way to construct a new \dd{} metric from a set of known \dd{} metrics. Note that not every success metric is decomposable; we can create non-decomposable success metrics by taking non-convex combinations of \dm s.

\subsection{Generalization of Bias}
\noindent Our definition of \dy{} allows us to redefine bias in terms of any \dd{} metric, $\phi(T, F)$. We replace $\overline{\textbf{P}}_F$ with $\textbf{P}_{\phi, F}$ and obtain
\begin{align}
\text{Bias}(\D, \textbf{t}) &= \mathbb{E}_{\D}[\textbf{t}^{\top}\textbf{P}_{\phi, F}] - \frac{k}{|\Omega|}. \\
&= \textbf{t}^{\top}\mathbb{E}_{\D}[\textbf{P}_{\phi, F}] - \frac{\|\textbf{t}\|^2}{|\Omega|}\\
&= \textbf{t}^{\top}\int_{\mathcal{F}} \textbf{P}_{\phi, f} \D(f) \text{d}f - \frac{\|\textbf{t}\|^2}{|\Omega|}.
\end{align}
Because $\phi(t, f)$ is \dd, it is equal to $\textbf{t}^{\top}\textbf{P}_{\phi, f}$. This makes results about the bias particularly interesting, since they relate directly to any probability-of-success metric we create, so long as the metric is \dd. 

\section{\uppercase{Results}}
\noindent \M{} proved a number of results and bounds on the success of machine learning algorithms relative to the \ep, along with its corresponding definition of bias~\cite{montanez2017machine,montanez2019fobfl}. We now generalize these to apply to any \dm, with full proofs given in the Appendix (available online).

\subsection{No Free Lunch for Search}
\noindent First, we prove a version of the No Free Lunch Theorems for any \dm{} within the search framework.
\begin{restatable}[No Free Lunch for Search and Machine Learning]{thm}{NFLthm}\label{thm:NFL}
For any pair of search/learning algorithms $\mathcal{A}_1$, $\mathcal{A}_2$ operating on discrete finite search space $\Omega$, any closed under permutation set of target sets $\uptau$, any set of information resources $\mathcal{B}$, and \dm{} $\g$,
     \begin{align}
         \sum_{t\in\uptau}\sum_{f\in\mathcal{B}} \g_{\mathcal{A}_1}(t,f) = \sum_{t\in\uptau}\sum_{f\in\mathcal{B}} \g_{\mathcal{A}_2}(t,f).
    \end{align}
\end{restatable}
\noindent This means that performance, in terms of our \dm{}, is conserved in the sense that increased performance of one algorithm over another on some information resource-target pair comes at the cost of a loss in performance elsewhere.

\subsection{The Fraction of Favorable Targets}
\noindent \M{} proved that for a fixed information resource, a given algorithm $\mathcal{A}$ will perform favorably relative to uniform random sampling on only a few target sets, under the \ep~\cite{montanez2017machine}. We generalize this result with a \dm{} and define a version of active information of expectations for decomposable metrics $I_{\g(t, f)} := -\log_2\frac{p}{\g(t, f)}$. This transforms the ratio of success probabilities into bits where $p = |t|/|\Omega|$, the per-query probability of success for uniform random sampling with replacement. $I_{\g(t,f)}$ denotes the advantage $\mathcal{A}$ has over uniform random sampling with replacement, in bits. 

\begin{restatable}[The Fraction of Favorable Targets]{thm}{FOFT}\label{thm:FOFT}
Let $\uptau = \{t \mid t \subseteq \Omega \}$, $\uptau_{b} = \{t \mid \emptyset \not= t \subseteq \Omega, I_{\g(t,f)} \geq b \}$, and \dm{} $\g$. Then for $b \geq 3$,
\begin{align}
    \frac{|\uptau_{b}|}{|\uptau|} \leq 2^{-b}.
\end{align}
\end{restatable}
\noindent Thus, the scarcity of $b$-bit favorable targets still holds under  for any \dm.  

\subsection{The Famine of Favorable Targets}
\noindent Following up on the previous result, we can show a similar bound in terms of the success of a given algorithm, for targets of a fixed size.

\begin{restatable}[The Famine of Favorable Targets]{thm}{FMOFT}\label{thm:FMOFT}
For fixed $k \in \N$, fixed information resource $f$, and \dm{} $\g$, define
\begin{align*}
\uptau &= \{T \mid T \subseteq \Omega, |T| = k\}, \text{ and } \\
\uptau_{q_{\text{min}}} &= \{T \mid T \subseteq \Omega, |T| = k, \g(T,f) \geq q_{\text{min}} \}.
\end{align*}
Then,
\begin{align}
    \frac{|\uptau_{q_{\text{min}}}|}{|\uptau|} \leq \frac{p}{q_{\text{min}}}
\end{align}
where $p = \frac{k}{|\Omega|}$.
\end{restatable}

 Here, we compare success not against uniform sampling but against a fixed constant $q_{\text{min}}$. This theorem thus upper bounds the number of targets for which the probability of success of the search is greater than $q_{\text{min}}$. 

\subsection{Famine of Forte}
\noindent We generalize the Famine of Forte ~\cite{montanez2017machine}, showing a bound that holds in the $k$-sparse case using any \dm{}. 
\begin{restatable}[The Famine of Forte]{thm}{FOF}\label{thm:FOF}
    Define 
\[
    \uptau_{k} = \{T \mid T \subseteq \Omega, |T| = k \in \N \}
\] 
and let $\mathcal{B}_m$ denote any set of binary strings, such that the strings are of length $m$ or less. Let 
\begin{align*}
    R &= \{(T, F) \mid T \in \uptau_{k}, F \in \mathcal{B}_m \}, \text{ and } \\
    R_{q_{\text{min}}} &= \{(T, F) \mid T \in \uptau_{k}, F \in \mathcal{B}_m, \g(T,F) \geq q_{\text{min}} \},
\end{align*}
where $\g(T,F)$ is the \dm{} for algorithm $\mathcal{A}$ on problem $( \Omega, T, F)$. Then for any $m \in \N$,
\begin{align}
    \frac{|R_{q_{\text{min}}}|}{|R|} &\leq \frac{p}{q_{\text{min}}}.
\end{align}
\end{restatable}

This demonstrates that for any \dd{} metric there is an upper bound on the proportion of problems an algorithm is successful on.  Here, we measure success as being above a certain threshold with respect to a \dd{} metric, and the upper bound is inversely related to this threshold. 

\subsection{Learning Under Dependence}
\noindent While the previous theorems highlight cases where an algorithm is unlikely to succeed, we now consider the conditions that make an algorithm likely to succeed. To begin, we consider how the target and information resource can influence an algorithm's success by generalizing the Learning Under Dependence theorem~\cite{Montanez2016TheFO}.

\begin{restatable}[Learning Under Dependence]
{thm}{LUD}\label{thm:MUTUAL-INFORMATION-BOUND}
Define $\uptau_{k} = \{T \mid T \subseteq \Omega, |T| = k \in \N\}$ and let $\mathcal{B}_m$ denote any set of binary strings (information resources), such that the strings are of length $m$ or less. Define $q$ as the expected \dd{} probability of success under the joint distribution on $T \in \uptau_{k}$ and $F \in \mathcal{B}_m$ for any fixed algorithm $\mathcal{A}$, such that $q := \E_{T, F}\left[\g(T,F)\right]$, namely,
\[
    q = \E_{T, F}\left[\;P_\g(\omega \in T | F) \right] = \Pr(\omega \in T; \mathcal{A}).
\]
Then, 
\begin{align}
        q \leq \frac{I(T; F) + D(P_T \| \mathcal{U}_T) + 1}{I_{\Omega}}
    \end{align}
    where $I_{\Omega} = -\log k/|\Omega|$, $D(P_T \| \mathcal{U}_T)$ is the Kullback-Liebler divergence between the marginal distribution on $T$ and the uniform distribution on $T$, and $I(T; F)$ is the mutual information. Alternatively, we can write
\begin{align}
    \Pr(\omega \in T; \mathcal{A}) \leq \frac{H(\mathcal{U}_T) - H(T \mid F) + 1}{I_{\Omega}}
\end{align}
where $H(\mathcal{U}_T) = \log\binom{|\Omega|}{k}$.
\end{restatable}

The value of $q$ defined here represents the expected single-query probability of success of an algorithm relative to a randomly selected target and information resource, distributed according to some joint distribution. The probability of success for a single query (marginalized over information resources) is equivalent to the expectation of the conditional probability of success, conditioned on the random information resource. Upper bounding this value states that regardless of the choice of \dm, the probability of success depends on the amount of information regarding the target contained within the information resource, as measured by the mutual information. 

\subsection{Famine of Favorable Information Resources}
\noindent We now demonstrate the effect of the general bias term defined earlier on the probability of a success of an algorithm. We begin with a generalization of the Famine of Favorable Information Resources~ \cite{montanez2019fobfl}. 

\begin{restatable}[Famine of Favorable Information Resources]{thm}{fofir}\label{thm:fofir}
Let $\mathcal{B}$ be a finite set of information resources and let $t \subseteq \Omega$ be an arbitrary fixed $k$-size target set with corresponding target function $\mathbf{t}$. Define 
        \begin{align*}
            \mathcal{B}_{q_{\mathrm{min}}} &= \{f \mid f \in \mathcal{B}, \g(t,f) \geq q_{\mathrm{min}} \},
        \end{align*}
        where $\g(t,f)$ is an arbitrary \dm{} for algorithm $\mathcal{A}$ on search problem $(\Omega, t,f)$ and $q_{\mathrm{min}} \in (0,1]$ represents the minimally acceptable probability of success. Then,
        \begin{align}
            \frac{|\mathcal{B}_{q_{\mathrm{min}}}|}{|\mathcal{B}|} &\leq \frac{p +  \mathrm{Bias}(\mathcal{B}, \mathbf{t})}{q_{\mathrm{min}}}
        \end{align}
        where $p = \frac{k}{|\Omega|}$.
    \end{restatable}
This result demonstrates the mathematical effect of bias, of which we have previously provided one hypothetical example (car parking). Now, we can show that the bias of our expected information resources towards the target upper bounds the probability of a given information resource leading to a successful search.

\subsection{Futility of Bias-Free Search}
\noindent We can also use our definition of bias to generalize the Futility of Bias-Free Search \cite{montanez2017machine}, which demonstrates the inability of an algorithm to perform better than uniform random sampling without bias, defined with respect to the \ep. Our generalization proves that the theorem holds for bias defined with respect to any \dm. 

\begin{restatable}[Futility of Bias-Free Search]{thm}{futility}\label{thm:futilitybs}
    For any fixed algorithm $\mathcal{A}$, fixed target $t \subseteq \Omega$ with corresponding target function $\mathbf{t}$, and distribution over information resources $\mathcal{D}$, if $\bias(\mathcal{D}, \mathbf{t}) = 0$, then
    \begin{align}
        \Pr(\omega \in t; \mathcal{A}) &= p
    \end{align}
    where $\Pr(\omega \in t; \mathcal{A})$ represents the expected \dd{} probability of successfully sampling an element of $t$ using $\mathcal{A}$, marginalized over information resources $F \sim \mathcal{D}$, and $p$ is the single-query probability of success under uniform random sampling.
    \end{restatable}
    
This result demonstrates that, regardless of how we measure the success of an algorithm with respect to a decomposable metric, it cannot perform better than uniform random sampling without bias.

\subsection{Famine of Favorable Biasing Distributions}
\noindent \M{} proved that the percentage of minimally favorable distributions (biased over some threshold towards some specific target) is inversely proportional to the threshold value and directly proportional to the bias between the information resource and target function~\cite{montanez2017machine}. We will show that this scarcity of favorable biasing distributions holds, in general, for bias under any \dm.

 \begin{restatable}[Famine of Favorable Biasing Distributions]{thm}{fofbd}\label{thm:fofbd}
      Given a fixed target function $\mathbf{t}$, a finite set of information resources $\mathcal{B}$, a distribution over information resources $\mathcal{D}$, and a set $\mathcal{P} = \{\mathcal{D}\mid \mathcal{D} \in \mathbb{R}^{|\mathcal{B}|}, \sum_{f \in \mathcal{B}} \mathcal{D}(f) = 1 \}$ of all discrete $|\mathcal{B}|$-dimensional simplex vectors,
      \begin{equation}
        \frac{\mu(\mathcal{G}_{\mathbf{t}, q_\mathrm{min}})}{\mu(\mathcal{P})} \leq \frac{p + \bias(\mathcal{B}, \mathbf{t})}{q_\mathrm{min}}
      \end{equation}
      where $\mathcal{G}_{\mathbf{t}, q_\mathrm{min}} = \{\mathcal{D} \mid \mathcal{D} \in \mathcal{P}, \bias(\mathcal{D}, \mathbf{t}) \geq q_\mathrm{min}\}, p=\frac{k}{\Omega}$, and $\mu$ is Lebesgue measure.
    \end{restatable}
 This result shows that the more bias there is between our set of information resources $\mathcal{B}$ and the target function $\mathbf{t}$, the easier it is to find a minimally favorable distribution, and the higher the threshold for what qualifies as a minimally favorable distribution, the harder our search becomes. Thus, unless we want to suppose that we begin with a set of information resources already favorable towards our fixed target, finding a highly favorable distribution is difficult.

\section{\uppercase{Conclusion}}
\noindent Casting machine learning problems as search provides a common formalism within which to prove bounds and impossibility results for a wide variety of learning algorithms and tasks. In this paper, we introduce a property of probability-of-success metrics called \dy, and show that the \ep{} and \gp{} are \dd. To demonstrate the value of this property, we prove that a number of existing algorithmic search framework results continue to hold for all \dm s. These results provide a number of useful insights: we show that algorithmic performance is conserved with respect to all \dm s, favorable targets are scarce no matter your \dm, and that without the generalized bias defined here, an algorithm will not perform better than uniform random sampling.  

The goal of this work is to offer additional machinery within the search framework, allowing for more general application. Concretely, we can develop \dm s for problems concerned with the state of an algorithm at specific steps, and leverage existing results as a foundation for additional insight into those problems.

\section*{\uppercase{Acknowledgements}}

\noindent This work was supported by the Walter Bradley Center for Natural and Artificial Intelligence. We thank Dr. Robert J. Marks II (Baylor University) for providing support and feedback. We also thank Harvey Mudd College's Department of Computer Science for their continued resources and support. 

\bibliographystyle{apalike}
{\small
\bibliography{references}}

\begin{thebibliography}{}

\bibitem[Fano and Hawkins, 1961]{fano1961transmission}
Fano, R.~M. and Hawkins, D. (1961).
\newblock Transmission of information: A statistical theory of communications.
\newblock {\em American Journal of Physics}, 29(11):793--794.

\bibitem[G{\"u}l{\c{c}}ehre and Bengio, 2016]{bengio}
G{\"u}l{\c{c}}ehre, {\c{C}}. and Bengio, Y. (2016).
\newblock Knowledge matters: Importance of prior information for optimization.
\newblock {\em The Journal of Machine Learning Research}, 17(1):226--257.

\bibitem[Kumagai and Kanamori, 2019]{KumagaiBounds}
Kumagai, W. and Kanamori, T. (2019).
\newblock Risk bound of transfer learning using parameric feature mapping and
  its application to sparse coding.
\newblock {\em Machine Learning}, 108:1975--2008.

\bibitem[Mitchell, 1980]{mitchell80}
Mitchell, T.~M. (1980).
\newblock The need for biases in learning generalizations.
\newblock Technical report, Computer Science Department, Rutgers University,
  New Brunswick, MA.

\bibitem[Mitchell, 1982]{mitchell1982generalization}
Mitchell, T.~M. (1982).
\newblock Generalization as {S}earch.
\newblock {\em Artificial intelligence}, 18(2):203--226.

\bibitem[{Monta{\~n}ez}, 2017]{Montanez2016TheFO}
{Monta{\~n}ez}, G.~D. (2017).
\newblock The {F}amine of {F}orte: {F}ew {S}earch {P}roblems {G}reatly {F}avor
  {Y}our {A}lgorithm.
\newblock In {\em Systems, Man, and Cybernetics (SMC), 2017 IEEE International
  Conference on}, pages 477--482. IEEE.

\bibitem[Monta{\~{n}}ez, 2017]{montanez2017machine}
Monta{\~{n}}ez, G.~D. (2017).
\newblock {\em Why {M}achine {L}earning {W}orks}.
\newblock PhD thesis, Carnegie Mellon University.

\bibitem[Monta{\~{n}}ez et~al., 2019]{montanez2019fobfl}
Monta{\~{n}}ez, G.~D., Hayase, J., Lauw, J., Macias, D., Trikha, A., and
  Vendemiatti, J. (2019).
\newblock The {F}utility of {B}ias-{F}ree {L}earning and {S}earch.
\newblock {\em CoRR}, abs/1907.06010.

\bibitem[Sauer, 1972]{SAUER-SHELAH}
Sauer, N. (1972).
\newblock On the density of families of sets.
\newblock {\em Journal of Combinatorial Theory, Series A}, 13(1):145--147.

\bibitem[Schaffer, 1994]{SchafferConservation}
Schaffer, C. (1994).
\newblock A {C}onservation {L}aw for {G}eneralization {P}erformance.
\newblock {\em Machine Learning Proceedings 1994}, 1:259--265.

\bibitem[Sterkenburg, 2019]{sterk}
Sterkenburg, T.~F. (2019).
\newblock Putnam's {D}iagonal {A}rgument and the {I}mpossibility of a
  {U}niversal {L}earning {M}achine.
\newblock {\em Erkenntnis}, 84(3):633--656.

\bibitem[Valiant, 1984]{ValiantLearnability}
Valiant, L. (1984).
\newblock A {T}heory of the {L}earnable.
\newblock {\em Communications of the ACM}, 27:1134--1142.

\bibitem[Wolpert and Macready, 1997]{Wolpert1997NoFL}
Wolpert, D.~H. and Macready, W.~G. (1997).
\newblock No free lunch theorems for optimization.
\newblock {\em IEEE Trans. Evolutionary Computation}, 1:67--82.

\end{thebibliography}

\section{Appendix}
\noindent Lemma \ref{SAUER-SHELAH}, Lemma \ref{BINOMIAL-APPROX}, Lemma \ref{SUBSET-SELECTION-LEMMA}, and Lemma \ref{GPSuccess} with their proofs are taken from \cite{montanez2017machine}, with  Lemma~\ref{GPSuccess} being adapted for \dm s. 
\subsection{Lemmata}
\begin{lem}[Sauer-Shelah Inequality]\label{SAUER-SHELAH}
    For $d \leq n$, $\sum_{j=0}^{d}\binom{n}{j} \leq \left(\frac{en}{d}\right)^d$.
\end{lem}

\begin{proof}
    We reproduce a simple proof of the Sauer-Shelah inequality~\cite{SAUER-SHELAH} for completeness.
\begin{align*}
    \sum_{j=0}^{d}\binom{n}{j} &\leq \left(\frac{n}{d}\right)^d\sum_{j=0}^{d}\binom{n}{j}\left(\frac{d}{n}\right)^j\\
      &\leq \left(\frac{n}{d}\right)^d\sum_{j=0}^{n}\binom{n}{j}\left(\frac{d}{n}\right)^j\\
      &= \left(\frac{n}{d}\right)^d\left(1+\frac{d}{n}\right)^n\\
      &\leq \left(\frac{n}{d}\right)^d \lim_{n\rightarrow\infty}\left(1+\frac{d}{n}\right)^n\\
      &= \left(\frac{en}{d}\right)^d.
\end{align*}
\end{proof}

\begin{lem}[Binomial Approximation]\label{BINOMIAL-APPROX}
        $\sum_{j=0}^{\left\lfloor\frac{n}{2^b}\right\rfloor}\binom{n}{j} \leq 2^{n-b}$ for $b \geq 3$ and $n \geq 2^b$.
\end{lem}

\begin{proof}  
    By the condition $b \geq 3$, we have 
    \begin{align*}
        2b + \log_2e &\leq 2^b
    \end{align*}
     which implies $2^{-b}(2b + \log_2e) \leq 1$. Therefore,
    \begin{align*}
        1 &\geq 2^{-b}(2b + \log_2e) \\
          &= \frac{b}{2^{b}} + \frac{b + \log_2e}{2^b} \\
          &\geq \frac{b}{n} + \frac{b + \log_2e}{2^b},
    \end{align*}
    using the condition $n \geq 2^b$, which implies
    \begin{align*}
        n &\geq b + \frac{n}{2^b}(b + \log_2 e).
    \end{align*}
Thus,
    \begin{align*}
        2^n &\geq 2^{b + \frac{n}{2^b}(b + \log_2 e)} \\
            &= 2^b 2^{\frac{n}{2^b} (b + \log_2e)} \\
            &= 2^b \left(2^b2^{\log_2e}\right)^{\frac{n}{2^b}} \\
            &= 2^b \left(2^b e\right)^{\frac{n}{2^b}} \\
            &= 2^b \left(\frac{en}{\frac{n}{2^b}}\right)^{\frac{n}{2^b}} \\
            &\geq 2^b \sum_{j = 0}^{\frac{n}{2^b}} \binom{n}{j} \\
            &\geq 2^b \sum_{j = 0}^{\floor{\frac{n}{2^b}}} \binom{n}{j},
    \end{align*}
    where the penultimate inequality follows from the Sauer-Shelah inequality~\cite{SAUER-SHELAH}. Dividing through by $2^{b}$ gives the desired result.
\end{proof}

\ifx
\begin{lem}(General Per Query Performance from Expected Distribution)\label{EXPECTED-PER-QUERY-PERFORMANCE}
Let $\A$ be a search algorithm. Define $t$ as a $n$-time-step moving target, $q_{\alpha}(t,f)$ as the general per query probability of success for $\A$, and $\nu$ be the conditional joint measure induced by that algorithm over finite sequences of probability distributions and search histories, conditioned on external information resource $f$. Let $P_i$ be the $i$th distribution in the probability distribution sequence $\tilde{P}$ and denote a search history as $h$. \newline 
    Let $P_{\alpha}$ be an arbitrary probability distribution, where 
    \[ \mathbb{E}_{P\sim P_{\alpha}}[P(x)] = \sum_{i=1}^{|\tilde{P}|} \alpha_i P_i(x) \]
    
    Then, 
    \[ q_{\alpha}(t,f) =  P_{\alpha}(X \in t|f)  \]
    
\end{lem}

\begin{proof}
Since $P(x)$ is a probability distribution, \newline

 \begin{align*}
        &P_{\alpha}(X \in t|f)  \\
        &= \sum_x \1_{x \in t} \int \mathbb{E}_{P\sim P_{\alpha}}[P(x)]d\nu(\tilde{P},h|f) \\
        &= \sum_x \1_{x \in t} \int \sum_{i=1}^{|\tilde{P}|} \alpha_i P_i(x)d\nu (\tilde{P},h|f)\\
        &= \int \sum_{i=1}^{|\tilde{P}|} \alpha_i \sum_x \1_{x \in t} P_i(x)d\nu (\tilde{P},h|f)\\
        &= \mathbb{E}_{\tilde{P}} \left[\sum_{i=1}^{|\tilde{P}|} P_i(x \in t | f) \right] \\
        &= q_{\alpha}(t,f) = P_{\alpha}^T t \\
\end{align*}
\end{proof}
\fi
\begin{lem}\label{SUBSET-SELECTION-LEMMA}(Maximum Number of Satisfying Vectors)
    Given an integer $1 \leq k \leq n$, a set $\mathcal{S} = \{\mathbf{s} : \mathbf{s} \in \{0, 1\}^n, \|\mathbf{s}\| = \sqrt{k}\}$ of all $n$-length $k$-hot binary vectors, a set $\mathcal{P} = \{\mathbf{P} : \mathbf{P} \in \mathbb{R}^n, \sum_j \mathbf{P}_j = 1\}$ of discrete $n$-dimensional simplex vectors, and a fixed scalar threshold $\epsilon \in [0, 1]$, then for any fixed $\mathbf{P} \in \mathcal{P}$,
\[
    \sum_{\mathbf{s}\in \mathcal{S}} \1_{\mathbf{s}^\top \mathbf{P} \geq \epsilon} \leq \frac{1}{\epsilon}\binom{n-1}{k-1}
\]
where $\mathbf{s}^\top \mathbf{P}$ denotes the vector dot product between $\mathbf{s}$ and $\mathbf{P}$. 
\end{lem}

\begin{proof}
    For $\epsilon = 0$, the bound holds trivially. For $\epsilon > 0$, let $\mathbf{S}$ be a random quantity that takes values $\mathbf{s}$ uniformly in the set $\mathcal{S}$. Then, for any fixed $\mathbf{P} \in \mathcal{P}$,
\begin{align*}
    \sum_{\mathbf{s}\in \mathcal{S}} \1_{\mathbf{s}^\top \mathbf{P} \geq \epsilon} &= \binom{n}{k} \mathbb{E}\left[\1_{\mathbf{S}^\top \mathbf{P} \geq \epsilon}\right] \\
          &= \binom{n}{k} \Pr\left(\mathbf{S}^\top \mathbf{P} \geq \epsilon \right).
\end{align*}
    Let $\mathbf{1}$ denotes the all ones vector. Under a uniform distribution on random quantity $\mathbf{S}$ and because $\mathbf{P}$ does not change with respect to $\mathbf{s}$, we have
\begin{align*}
    \mathbb{E}\left[\mathbf{S}^\top \mathbf{P} \right] &= \binom{n}{k}^{-1}\sum_{\mathbf{s}\in \mathcal{S}}\mathbf{s}^\top \mathbf{P} \\
         &= \mathbf{P}^\top\binom{n}{k}^{-1}\sum_{\mathbf{s}\in \mathcal{S}}\mathbf{s} \\
         &= \mathbf{P}^\top\frac{\mathbf{1}\binom{n-1}{k-1}}{\binom{n}{k}} \\
         &= \mathbf{P}^\top\frac{\mathbf{1}\binom{n-1}{k-1}}{\frac{n}{k}\binom{n-1}{k-1}} \\
         &= \frac{k}{n}\mathbf{P}^\top\mathbf{1} \\
         &= \frac{k}{n}
\end{align*}
since $\mathbf{P}$ must sum to $1$.

Noting that $\mathbf{S}^\top \mathbf{P} \geq 0$, we use Markov's inequality to obtain
\begin{align*}
\sum_{\mathbf{s}\in \mathcal{S}} \1_{\mathbf{s}^\top \mathbf{P} \geq \epsilon}
          &= \binom{n}{k} \Pr\left(\mathbf{S}^\top \mathbf{P} \geq \epsilon \right) \\
          &\leq \binom{n}{k} \frac{1}{\epsilon}\mathbb{E}\left[\mathbf{S}^\top \mathbf{P}\right] \\
          &= \binom{n}{k} \frac{1}{\epsilon} \frac{k}{n} \\
          &= \frac{1}{\epsilon}\frac{k}{n} \frac{n}{k}\binom{n-1}{k-1} \\
          &= \frac{1}{\epsilon}\binom{n-1}{k-1}.
\end{align*}
\end{proof}
\begin{lem}\label{GPSuccess}
If $X \perp T|F$, then
\[
\Pr(X \in T; \mathcal{A}) = \E_{T, F}[\g (T, F)].
\]

\end{lem}

\begin{proof}
$\Pr(X \in T; \mathcal{A})$ is the probability that random variable $X$ is in target $T$ over all values of $F$, for random $T$ and for $X$ drawn from $P_\g (X|F)$. Then, 
\begin{align*}
    \Pr(X \in T; \mathcal{A}) &= \E_{T, X} [\1_{X \in T} \mid \mathcal{A}] \\
    &= \E_T[\E_X[\1_{X \in T} \mid T, \mathcal{A}]] \\
    &= \E_T[\E_F[\E_X[\1_{X \in T} \mid T, F, \mathcal{A}]\mid T]] \\
    &= \E_T[\E_F[\E_X[\1_{X \in T} \mid F, \mathcal{A}]\mid T]] \\
    &= \E_{T, F}[\E_X[\1_{X \in T} \mid F, \mathcal{A}]] \\
    &= \E_{T, F}[P_\g(X \in T | F)] \\
    &= \E_{T, F}[\g(T, F)],
\end{align*}
where the third equality makes use of the law of iterated expectation, the fourth follows from the
conditional independence assumption, and the final equality follows from the definition of decomposability.
\end{proof}

\DEPlem*
\begin{proof}
By definition,
\[
q(t,f) = \overline{P}(X \in t|f) = \mathbf{t}^\top \overline{\mathbf{P}}_f.
\]
\end{proof}

\DGPlem*
\begin{proof}
Observe that $P_{\alpha}(X \in t|f)  = \textbf{t}^\top \textbf{P}_{\alpha, f}$. Since $P(x)$ is a probability distribution, 
 \begin{align*}
        P_{\alpha}(X \in t|f) &= \sum_x \1_{x \in t} \int \mathbb{E}_{P\sim P_{\alpha}}[P(x)]\dif\nu(\tilde{P},h|f) \\
        &= \sum_x \1_{x \in t} \int \sum_{i=1}^{|\tilde{P}|} \alpha_i P_i(x)\dif\nu (\tilde{P},h|f)\\
        &= \int\left[ \sum_{i=1}^{|\tilde{P}|} \alpha_i \left(\sum_x \1_{x \in t} P_i(x)\right)\right]\dif\nu (\tilde{P},h|f)\\
        &= \mathbb{E}_{\tilde{P},H} \left[\sum_{i=1}^{|\tilde{P}|}\alpha_i P_i(x \in t) \bigg| f \right] \\
        &= q_{\alpha}(t,f).
\end{align*}
\end{proof}

\DCElem*
\begin{proof}
\begin{align*}
     \g'(t,f) = \E_{\mathcal{D}}[\g(t,f)] &= \sum_{i}\g_i(t,f)\mathcal{D}(\g_i)\\
     &= \sum_{i} \left(\textbf{t}^\top \textbf{P}_{\g_i,f}\right)\mathcal{D}(\g_i)\\
     &= \textbf{t}^\top  \sum_{i} \textbf{P}_{\g_i,f}\mathcal{D}(\g_i)\\
     &= \textbf{t}^\top \E_{\mathcal{D}}[\textbf{P}_{\g,f}]\\
     &= \textbf{t}^\top \overline{\textbf{P}}_{\g,f}.
\end{align*}
\end{proof}

\subsection{Proofs of Theorems}
\NFLthm*

\begin{proof}
Note that the closed under permutation condition implies $\sum_{\mathbf{t}\in\uptau}\mathbf{t} = [c, c,\ldots,c] = \mathbf{1}\cdot c$ for some constant $c$.

\begin{align*}
            \sum_{t\in\uptau}\sum_{f\in\mathcal{B}} \g_{\mathcal{A}_1}(t,f) 
                &= \sum_{\mathbf{t}\in\uptau}\sum_{f\in\mathcal{B}} \mathbf{P}_{\g, f,\mathcal{A}_1}^{\top}\mathbf{t}\\
                &= \sum_{f\in\mathcal{B}} \textbf{P}_{\g, f,\mathcal{A}_1}^{\top}\sum_{\mathbf{t}\in\uptau}\mathbf{t}\\
                &= \sum_{f\in\mathcal{B}} \textbf{P}_{\g, f,\mathcal{A}_1}^{\top}\mathbf{1} \cdot c\\
                &= c \sum_{f\in\mathcal{B}} \textbf{P}_{\g, f,\mathcal{A}_1}^{\top}\mathbf{1}\\
                &= c \sum_{f\in\mathcal{B}} 1 \\
                &= c \sum_{f\in\mathcal{B}} \textbf{P}_{\g, f,\mathcal{A}_2}^{\top}\mathbf{1}\\
                &= \sum_{\mathbf{t}\in\uptau}\sum_{f\in\mathcal{B}} \mathbf{P}_{\g, f,\mathcal{A}_2}^{\top}\mathbf{t}\\
                &= \sum_{t\in\uptau}\sum_{f\in\mathcal{B}}\g_{\mathcal{A}_2}(t,f)
        \end{align*}
where the first and final equalities follow from the definition of \dy. 
\end{proof}

\FOFT*
\begin{proof}
First, by the definition of active information of expectations, $I_{\g(T,F)} \geq b$ implies $|T| \leq \frac{|\Omega|}{2^b}$, since $\g(T,F) \leq 1$. Thus, 
\begin{align}
    \uptau_b &\subseteq \uptau'_b = \left\{T \mid T \subseteq \Omega, 1 \leq |T| \leq \frac{|\Omega|}{2^b} \right\}.
\end{align}

For $|\Omega| < 2^b$ and $I_{\g(T,F)} \geq b$, we have $|T| < 1$ for all elements of $\uptau'_b$ (making the set empty) and the theorem follows immediately. Thus, $|\Omega| \geq 2^b$ for the remainder. 

By Lemma~\ref{BINOMIAL-APPROX}, we have
\begin{align}
        \frac{|\uptau_b|}{|\uptau|} &\leq \frac{|\uptau'_b|}{|\uptau|}\\
        &= 2^{-|\Omega|}\sum_{k=0}^{\left\lfloor\frac{|\Omega|}{2^b}\right\rfloor}\binom{|\Omega|}{k}\\
        &\leq 2^{-|\Omega|}2^{|\Omega| - b}\\
        &= 2^{-b}. 
    \end{align}
\end{proof}

\FMOFT*
\begin{proof}
Let $\mathcal{S} = \{\mathbf{s} : \mathbf{s} \in \{0, 1\}^{|\Omega|}, \|\mathbf{s}\| = \sqrt{k}\}$. For brevity, we will allow $\mathbf{s}$ to also denote its corresponding target set, letting the context make clear whether the target set or target function is meant. Then,
\begin{align*}
    \frac{|\uptau_{q_{\text{min}}}|}{|\uptau|} &= \frac{\sum_{\mathbf{s} \in \mathcal{S}} \1_{\g(\mathbf{s}, f) \geq q_\text{min}} }{\binom{|\Omega|}{k}} \\
    &= \binom{|\Omega|}{k}^{-1}\sum_{\mathbf{s} \in \mathcal{S}} \1_{\g(\mathbf{s}, f) \geq q_\text{min}} \\
    &= \E_{\mathcal{U}[\mathcal{S}]}\left[\1_{\g(\mathbf{S}, f) \geq q_\text{min}} \right]\\
    &= \Pr(\g(\mathbf{S}, f) \geq q_\text{min})\\
    &\leq \frac{\E_{\mathcal{U}[\mathcal{S}]}[\g(\mathbf{S}, f)]}{q_\text{min}},
\end{align*}
where the final step follows from Markov's inequality. By decomposability of $\g$ and linearity of expectation, we have
\begin{align*}
     \frac{\E_{\mathcal{U}[\mathcal{S}]}[\g(\mathbf{S}, f)]}{q_\text{min}} &= \frac{\E_{\mathcal{U}[\mathcal{S}]}[\mathbf{S}^{\top}\textbf{P}_{\g,f}]}{q_\text{min}}\\
    &= \frac{\textbf{P}_{\g,f}^{\top}\E_{\mathcal{U}[\mathcal{S}]}[\mathbf{S}]}{q_\text{min}}\\
    &= \frac{\textbf{P}_{\g,f}^{\top}\mathbf{1}\left[\binom{|\Omega|}{k}^{-1}\binom{|\Omega| - 1}{k -1} \right]}{q_\text{min}}\\
    &= \frac{k}{|\Omega|}\frac{\textbf{P}_{\g,f}^{\top}\mathbf{1}}{q_\text{min}}\\
    &= \frac{p}{q_\text{min}}.
\end{align*}

\end{proof}

\FOF*
\begin{proof}
We begin by defining a set $\mathcal{S}$ of all $|\Omega|$-length target functions with exactly $k$ ones, namely, $\mathcal{S} = \{\mathbf{s} : \mathbf{s} \in \{0, 1\}^{|\Omega|}, \|\mathbf{s}\| = \sqrt{k}\}$. As in Theorem~\ref{thm:FMOFT}, we again allow $\mathbf{s}$ to also denote its corresponding target set. For each of these $\mathbf{s}$, we have $|\mathcal{B}_m|$ information resources. The total number of search problems is therefore
\begin{align}\label{NUMBER-OF-SEARCH-PROBLEMS}
    \binom{|\Omega|}{k}|\mathcal{B}_m|.
\end{align}
We seek to bound the proportion of possible search problems for which $\g(\mathbf{s}, f) \geq q_\text{min}$ for any threshold $q_\text{min} \in (0, 1]$. Thus,
\begin{align}
    \frac{|R_{q_{\text{min}}}|}{|R|} &\leq \frac{|\mathcal{B}_m| \underset{f}{\text{ sup}} \left[\sum_{\mathbf{s} \in \mathcal{S}} \1_{\g(\mathbf{s}, f) \geq q_\text{min}}\right] }{|\mathcal{B}_m|\binom{|\Omega|}{k}} \\
                  &= \binom{|\Omega|}{k}^{-1}\sum_{\mathbf{s} \in \mathcal{S}} \1_{\g(\mathbf{s}, f^*) \geq q_\text{min}}
\end{align}
where $f^* \in \mathcal{B}_m$ denotes the arg sup of the expression. Therefore,
\begin{align*}
    \frac{|R_{q_{\text{min}}}|}{|R|} &\leq \binom{|\Omega|}{k}^{-1}\sum_{\mathbf{s} \in \mathcal{S}} \1_{\g(\mathbf{s}, f^*) \geq q_\text{min}} \\
        &= \binom{|\Omega|}{k}^{-1}\sum_{\mathbf{s} \in \mathcal{S}} \1_{\mathbf{s}^\top \textbf{P}_{\g, f^*} \geq q_\text{min}}
\end{align*}
where the equality follows decomposability of $\g(\mathbf{s}, f^*)$ and $\textbf{P}_{\g, f^*}$ represents the $|\Omega|$-length probability vector defined by $P_\g(\cdot|f^*)$. By Lemma~\ref{SUBSET-SELECTION-LEMMA}, we have
\begin{align}
\binom{|\Omega|}{k}^{-1}\sum_{\mathbf{s} \in \mathcal{S}} \1_{\mathbf{s}^\top \mathbf{P}_{\g, f^*} \geq q_\text{min}} 
&\leq \binom{|\Omega|}{k}^{-1} \left[\frac{1}{q_\text{min}}\binom{|\Omega|-1}{k-1}\right] \notag{}\\
&= \frac{k}{|\Omega|}\frac{1}{q_\text{min}} \notag{}\\
&= p / q_\text{min}\label{eq:final-bound-thm-pgsp}
\end{align}
proving the result for finite information resources.

\end{proof}

\LUD*

\begin{proof}
This proof loosely follows that of Fano's Inequality~\cite{fano1961transmission}, being a reversed generalization of it. Let $Z = \1(\X \in T)$. Using the chain rule for entropy to expand $H(Z, T | \X)$ in two different ways, we get
\begin{align*}
    H(Z, T | \X)  &= H(Z|T, \X) + H(T|\X) \\
                           &= H(T|Z, \X) + H(Z|\X).
\end{align*}
By definition, $H(Z|T, \X) = 0$, and by the data processing inequality $H(T|F) \leq H(T|\X)$. Thus,
\begin{align*}
    H(T | F)  &\leq H(T|Z, \X) + H(Z|\X).
\end{align*}
Define $P_g = \Pr(\X \in T ; \mathcal{A}) = P(Z=1)$. Then,
\begin{align*}
H(T|Z, \X) &= (1-P_g)H(T|Z=0, \X) + P_gH(T|Z=1, \X) \\
                    &\leq (1-P_g)\log\binom{|\Omega|}{k} + P_g\log\binom{|\Omega|-1}{k-1} \\
                    &= \log\binom{|\Omega|}{k} - P_g\log\frac{|\Omega|}{k}.
\end{align*}
We let $H(\mathcal{U}_T) = \log \binom{|\Omega|}{k}$, being the entropy of the uniform distribution over $k$-sparse target sets in $\Omega$. Therefore,
\begin{align*}
    H(T|F) &\leq H(\mathcal{U}_T) - P_g\log\frac{|\Omega|}{k} + H(Z|\X).
\end{align*}
Using the definitions of conditional entropy and $I_{\Omega}$, we get
\begin{align*}
    H(T) - I(T; F) &\leq H(\mathcal{U}_T) - P_gI_{\Omega} + H(Z|\X),
\end{align*}
which implies
\begin{align*}
     P_gI_{\Omega} &\leq I(T; F) + H(\mathcal{U}_T) - H(T) + H(Z|\X) \\
        &= I(T; F) + D(P_T \| \mathcal{U}_T) + H(Z|\X).
\end{align*}
Examining $H(Z|\X)$, we see it captures how much entropy of $Z$ is due to the randomness of $T$.
\ifx 
To see this, imagine $\Omega$ is a roulette wheel and we place our bet on $\X$. Target elements are ``chosen'' as balls land on random slots, according to the distribution on $T$. When a ball lands on $\X$ as often as not (roughly half the time), this quantity is maximized. Thus, this entropy captures the contribution of dumb luck, being averaged over all $\X$. (When balls move towards always landing on $\X$, something other than luck is at work.) 
\fi 
We upperbound this by its maximum value of 1 and obtain
\begin{align*}
     \Pr(\X \in T ; \mathcal{A}) &\leq \frac{I(T; F) + D(P_T \| \mathcal{U}_T) + 1}{I_{\Omega}},
\end{align*}
and substitute $q$ for $\Pr(\X \in T ; \mathcal{A})$ to obtain the first result, noting that $q = \E_{T, F}\left[\;P_\g(\omega \in T | F) \right]$ specifies a proper probability distribution by the linearity and boundedness of the expectation. To obtain the second form, use the definitions $I(T; F) = H(T) - H(T|F)$ and $D(P_T \| \mathcal{U}_T) = H(\mathcal{U}_T) - H(T)$.
\end{proof}

\fofir*
\begin{proof}
        We seek to bound the proportion of successful search problems for which $\g(t, f) \geq q_{\mathrm{min}}$ for any threshold $q_{\mathrm{min}} \in (0, 1]$. Let $F \sim \mathcal{U}[\mathcal{B}]$. Then, 
        \begin{align*}
            \frac{|\mathcal{B}_{q_{\mathrm{min}}}|}{|\mathcal{B}|} &= \frac{1}{ |\mathcal{B}|} \sum_{f \in \mathcal{B}} \mathds{1}_{\g(t,f) \geq q_{\mathrm{min}}}\\
                                               &=  \mathbb{E}_{\mathcal{U}[\mathcal{B}]}[\mathds{1}_{\g(t,F) \geq q_{\mathrm{min}}}] \\
                                               &= \Pr(\g(t, F) \geq q_{\mathrm{min}}).
        \end{align*}
        By \dy, we have
        \begin{align*}
            \frac{|\mathcal{B}_{q_{\mathrm{min}}}|}{|\mathcal{B}|} &= \Pr(\mathbf{t}^{\top} \mathbf{P}_{\g, F} \geq q_\mathrm{min}).
        \end{align*}
        Applying Markov's Inequality and by the definition of $\bias(\mathcal{B}, \mathbf{t})$, we obtain
        \begin{align*}
            \frac{|\mathcal{B}_{q_{\mathrm{min}}}|}{|\mathcal{B}|} &\leq \frac{\mathbb{E}_{\mathcal{U}[\mathcal{B}]} [\mathbf{t}^{\top} \textbf{P}_{\g, F}]}{q_{\mathrm{min}}} \\
                                               &= \frac{p + \bias(\mathcal{B},\mathbf{t})}{q_{\mathrm{min}}}.
        \end{align*}
    \end{proof}

\futility*
\begin{proof}
        Let $\mathcal{F}$ be the space of possible information resources. Then,
        \begin{align*}
            \Pr(\omega \in t; \mathcal{A}) 
                &= \int_\mathcal{F} \Pr(\omega \in t, f; \mathcal{A}) \dif f\\
                &= \int_\mathcal{F} \Pr(\omega \in t \mid f; \mathcal{A})\Pr(f) \dif f.
        \end{align*}
        Since we are considering the per-query probability of success for algorithm $\mathcal{A}$ on $t$ using information resource $f$, we have
        \[
          \Pr(\omega \in t \mid f; \mathcal{A}) = P_\g(\omega \in t \mid f).
        \]
        Also note that $\Pr(f) = \mathcal{D}(f)$ by the fact that $F \sim \mathcal{D}$. Making these substitutions, we obtain
        \begin{align*}
            \Pr(\omega \in t; \mathcal{A}) 
                &= \int_\mathcal{F} P_\g(\omega \in t \mid f)\mathcal{D}(f) \dif f\\
                &= \mathbb{E}_{\mathcal{D}}\left[P_\g(\omega \in t \mid F)\right]\\
                &= \mathbb{E}_{\mathcal{D}}\left[\mathbf{t}^{\top}\textbf{P}_{\g,F}\right]\\
                &= \bias(\mathcal{D}, \bm{t}) + p\\
                &= p.
        \end{align*}
    \end{proof}

\fofbd*
\begin{proof}
This result follows from \M's \cite{montanez2019fobfl} proof of the Famine of Favorable Biasing Distributions but instead using the generalized form of bias. No other changes to the proof are needed.
\end{proof}

\end{document}